\documentclass[number]{ReportTemplate}
\usepackage{utopia}
\usepackage{amssymb}
\usepackage{amsfonts}
\usepackage{amsmath}
\usepackage{mathtools}
\usepackage{amsthm}
\usepackage{bm}
\usepackage{graphicx}
\usepackage{algorithm}
\usepackage{algorithmic}
\usepackage[american]{babel}
\usepackage{subfigure}
\usepackage{hyperref}
\mathtoolsset{showonlyrefs,showmanualtags}

\renewenvironment{myproof}[1]
{\par\noindent\textbf{Proof of #1.}\ \enspace\ignorespaces\begin{allowdisplaybreaks}}
{\end{allowdisplaybreaks}\hspace{\stretch{1}}$\square$}

\begin{document}
\begin{frontmatter}

\title{Maximizing Submodular or Monotone Approximately Submodular Functions by Multi-objective Evolutionary Algorithms}
\author{Chao Qian$^{1}$}
\ead{qianc@lamda.nju.edu.cn}
\author{Yang Yu$^{1}$}
\ead{yuy@lamda.nju.edu.cn}
\author{Ke Tang$^{2}$}
\ead{tangk3@sustech.edu.cn}
\author{Xin Yao$^{2}$}
\ead{xiny@sustech.edu.cn}
\author{Zhi-Hua Zhou$^{1}$\corref{cor1}}
\ead{zhouzh@lamda.nju.edu.cn}
\cortext[cor1]{Corresponding author}
\address{$^{1}$National Key Laboratory for Novel Software Technology,\\ Nanjing University, Nanjing 210023, China\\\vspace{0.5em}
$^{2}$Shenzhen Key Laboratory of Computational Intelligence, Department of Computer Science and Engineering,\\ Southern University of Science and Technology, Shenzhen 518055, China}

\begin{abstract}
Evolutionary algorithms (EAs) are a kind of nature-inspired general-purpose optimization algorithm, and have shown empirically good performance in solving various real-word optimization problems. During the past two decades, promising results on the running time analysis (one essential theoretical aspect) of EAs have been obtained, while most of them focused on isolated combinatorial optimization problems, which do not reflect the general-purpose nature of EAs. To provide a general theoretical explanation of the behavior of EAs, it is desirable to study their performance on general classes of combinatorial optimization problems. To the best of our knowledge, the only result towards this direction is the provably good approximation guarantees of EAs for the problem class of maximizing monotone submodular functions with matroid constraints. The aim of this work is to contribute to this line of research. Considering that many combinatorial optimization problems involve non-monotone or non-submodular objective functions, we study the general problem classes, maximizing submodular functions with/without a size constraint and maximizing monotone approximately submodular functions with a size constraint. We prove that a simple multi-objective EA called GSEMO-C can generally achieve good approximation guarantees in polynomial expected running time.
\end{abstract}

\end{frontmatter}

\newpage
\section{Introduction}

Evolutionary algorithms (EAs)~\cite{back:96} are a kind of randomized metaheuristic optimization algorithm, inspired by the evolution process of natural species, i.e., natural selection and survival of the fittest. Starting from a random population of solutions, EAs iteratively apply reproduction (e.g., mutation and recombination) operators to generate offspring solutions from the current population, and then apply a selection operator to eliminate less desirable solutions. EAs have been applied to diverse areas (e.g., robotics~\cite{li2018path}, networks~\cite{yuan2018virtual} and machine learning~\cite{zhou2019evolutionary}) and can produce human-competitive results~\cite{koza2003s}. Compared with the application, the theoretical analysis of EAs is, however, far behind. Many researchers thus have been devoted to understanding the behavior of EAs from a theoretical point of view, which is still an ongoing challenge.

During the past two decades, a lot of progress has been made on the running time analysis of EAs, which is one essential theoretical aspect. The running time measures how many objective (i.e., fitness) function evaluations an EA needs until finding an optimal solution or an approximate solution. The running time analysis of EAs started with artificial example problems. In~\cite{droste1998rigorous,droste2002analysis}, a simple single-objective EA called (1+1)-EA has been shown to be able to solve two well-structured pseudo-Boolean problems OneMax and LeadingOnes in $\Theta(n\log n)$ and $\Theta(n^2)$ (where $n$ is the problem size) expected running time, respectively. These two problems are to maximize the number of 1-bits of a solution and the number of consecutive 1-bits counting from the left of a solution, respectively. Both of them have a short path with increasing fitness to the optimum. For some problems (e.g., SPC) where there is a short path with constant fitness to the optimum, the (1+1)-EA can also find an optimal solution in polynomial expected time~\cite{jansen2001evolutionary}. But when the problem (e.g., Trap) has a deceptive path, i.e., a path with increasing fitness away from the optimum, the (1+1)-EA will need exponential running time~\cite{he2001drift}. More results can be found in~\cite{auger2011theory}.

The analysis on simple artificial problems disclosed theoretical properties of EAs (e.g., which problem structures are easy or hard for EAs), and also helped to develop approaches for analyzing more complex problems. The running time analysis of EAs was then extended to combinatorial optimization problems. For some P-solvable problems, EAs have been shown to be able to find an optimal solution in polynomial expected time. For example, the minimum spanning tree problem can be solved by the (1+1)-EA and a simple multi-objective EA called GSEMO in $O(m^2(\log n+\log w_{\max}))$~\cite{neumann2007randomized} and $O(mn(n+\log w_{\max}))$~\cite{neumann2006minimum} expected time, respectively. Note that $m$, $n$ and $w_{\max}$ are the number of edges, the number of nodes and the maximum edge weight of a graph, respectively. For some NP-hard problems, EAs have been shown to be able to achieve good approximation ratios in polynomial expected time. For example, for the partition problem, the (1+1)-EA can achieve a $(4/3)$-approximation ratio in $O(n^2)$ expected time~\cite{witt2005worst}; for the minimum set cover problem, the expected running time of the GSEMO until obtaining a $(\log m +1)$-approximation ratio is $O(m^2n+mn(\log n+\log c_{\max}))$~\cite{friedrich2010approximating}, where $m$, $n$ and $c_{\max}$ denote the size of the ground set, the number of subsets and the maximum cost of a subset, respectively. For more running time results of EAs on combinatorial optimization problems, the reader can refer to~\cite{neumann2010bioinspired}.

For the analysis of the GSEMO (which is a multi-objective EA) on single-objective optimization problems (e.g., minimum spanning tree and minimum set cover), the original single-objective problem is transformed into a multi-objective problem, which is then solved by the GSEMO. Note that multi-objective optimization here is just an intermediate process, which might be beneficial~\cite{friedrich2010approximating,neumann2006minimum,neumann2011computing,qian.ijcai15}, and we still focus on the quality of the best solution w.r.t. the original single-objective problem, in the population found by the GSEMO. Running time analysis of EAs on real multi-objective optimization problems has also been investigated, where the running time is measured by the number of fitness evaluations until finding the Pareto front (which represents different optimal tradeoffs between the multiple objectives) or an approximation of the Pareto front. For example, Giel~\cite{GielCEC03} proved that the GSEMO can solve the bi-objective pseudo-Boolean problem LOTZ in $O(n^3)$ expected time; for the NP-hard bi-objective minimum spanning tree problem, it has been shown that the GSEMO can obtain a $2$-approximation ratio in pseudo-polynomial time~\cite{neumann2007expected,qian2013analysis}.

The analysis on combinatorial optimization problems helped to reveal the ability of EAs. However, most of the previous promising results were obtained for isolated problems, while EAs are known to be general-purpose optimization algorithms, which can be applied to various problems. Thus, it is more desirable to provide a general theoretical explanation of the behavior of EAs, that is, to theoretically study the performance of EAs on general classes of combinatorial optimization problems.

To the best of our knowledge, only two pieces of work in this direction have been reported. Reichel and Skutella~\cite{reichel2010evolutionary} first studied the problem class of maximizing linear functions with $k$ matroid constraints, which includes some well-known combinatorial optimization problems such as maximum matching, Hamiltonian path, etc. They proved that the (1+1)-EA can obtain a $(1/k)$-approximation ratio in $O(n^{k+2}(\log r+\log w_{\max}))$ expected running time, where $n$, $r$ and $w_{\max}$ denote the size of the ground set, the minimum rank of the ground set w.r.t. one matroid and the maximum weight of an element, respectively. Later, Friedrich and Neumann~\cite{friedrich2015maximizing} considered a more general problem class, where the objective function is relaxed to satisfy the monotone and submodular property. The (1+1)-EA has been shown to be able to achieve a $(\frac{1}{k+1/p+\epsilon})$-approximation ratio in $O(\frac{1}{\epsilon}n^{2p(k+1)+1}k\log n)$ expected time, where $p\geq 1$ and $\epsilon >0$. They also studied a specific non-monotone case, i.e., symmetric objective functions, and proved that the expected running time until the GSEMO obtains a $(\frac{1}{(k+2)(1+\epsilon)})$-approximation ratio for maximizing symmetric submodular functions with $k$ matroid constraints is $O(\frac{1}{\epsilon}n^{k+6}\log n)$.

The aim of this paper is to contribute to this line of research. Considering that the objective function of many combinatorial optimization problems can be non-monotone (not necessarily symmetric) or non-submodular, we study the performance of EAs on the general problem classes, maximizing submodular functions with/without a size constraint and maximizing monotone approximately submodular functions with a size constraint. Note that the objective function is a set function $f: 2^{V} \rightarrow \mathbb{R}$ which maps a subset of the ground set $V$ to a real value, and a size constraint means that the size of a subset is no larger than a budget $k$. We prove that for any concerned problem class, a variant of the GSEMO, called GSEMO-C, can obtain a good approximation guarantee in polynomial expected running time. Our main results can be summarized as follows.\vspace{-0.8em}
\begin{itemize}
  \item For the problem class of maximizing non-monotone submodular functions without constraints, with special instances including maximum cut~\cite{goemans1995improved}, maximum facility location~\cite{ageev19990} and variants of the maximum satisfiability problem~\cite{haastad2001some}, we prove that the GSEMO-C achieves a constant approximation ratio of $(\frac{1}{3}-\frac{\epsilon}{n})$ in $O(\frac{1}{\epsilon}n^4\log n)$ expected running time (i.e., {\bf Theorem~\ref{theo-nonmonotone}}), where $n$ is the size of the ground set $V$ and $\epsilon >0$.
  \item For the problem class of maximizing submodular and approximately monotone functions with a size constraint, with special instances such as sensor placement~\cite{krause2008near}, we prove that the GSEMO-C within $O(n^2(\log n+k))$ expected time finds a subset $X$ with $f(X) \geq (1-1/e)\cdot(\mathrm{OPT}-k\epsilon)$ (i.e., {\bf Theorem~\ref{theo-apprx-monotone}}), where $e$ is the base of the natural logarithm, $\mathrm{OPT}$ denotes the optimal function value, and $\epsilon \geq 0$ captures the degree of approximate monotonicity.
  \item For the problem class of maximizing monotone and approximately submodular functions with a size constraint, with special instances including sparse regression~\cite{das2011submodular}, dictionary selection~\cite{krause2010submodular} and Bayesian experimental design~\cite{krause2008near}, we prove the approximation guarantee of the GSEMO-C w.r.t. each notion of ``approximate submodularity", which measures how close a general set function $f$ is to submodularity.\\
      (1) \;In~\cite{krause2010submodular}, a set function $f$ is $\epsilon$-approximately submodular if the diminishing returns property holds with some deviation $\epsilon \geq 0$, i.e., for any $X \subseteq Y \subseteq V$ and $v \notin Y$, $f(X \cup \{v\})-f(X) \geq f(Y \cup \{v\})-f(Y)-\epsilon$. $f$ is submodular iff $\epsilon=0$. We prove that the GSEMO-C within $O(n^2(\log n+k))$ expected time finds a subset $X$ with $f(X) \geq (1-1/e)\cdot(\mathrm{OPT}-k\epsilon)$ (i.e., {\bf Theorem~\ref{theo-nonsubmodular-1}}).\\
      (2) \;In~\cite{das2011submodular}, the approximately submodular degree of a set function $f$ is characterized by a quantity $\gamma$ called submodularity ratio. $f$ is submodular iff $\gamma=1$. We prove that the GSEMO-C within $O(n^2(\log n+k))$ expected time finds a subset $X$ with $f(X) \geq (1-e^{-\gamma})\cdot \mathrm{OPT}$ (i.e., {\bf Theorem~\ref{theo-nonsubmodular-2}}).\\
      (3) \;In~\cite{horel2016maximization}, a set function $f$ is $\epsilon$-approximately submodular if there exists a submodular set function $g$ such that $\forall X \subseteq V$, $(1-\epsilon)g(X)\leq f(X)\leq (1+\epsilon)g(X)$. $f$ is submodular iff $\epsilon=0$. We prove that the GSEMO-C within $O(n^2(\log n+k))$ expected time finds a subset $X$ with $f(X) \geq \frac{1}{1+\frac{2k\epsilon}{1-\epsilon}}(1-e^{-1}(\frac{1-\epsilon}{1+\epsilon})^{k})\cdot \mathrm{OPT}$ (i.e., {\bf Theorem~\ref{theo-nonsubmodular-3}}).\vspace{-0.8em}
\end{itemize}

Because EAs are general-purpose algorithms which utilize a small amount of problem knowledge, we cannot expect them to beat the best problem-specific algorithm. For maximizing non-monotone submodular functions without constraints, the approximation ratio of nearly $1/3$ obtained by the GSEMO-C is worse than the best known one $1/2$, which was previously obtained by the double greedy algorithm~\cite{buchbinder2015tight}. For maximizing submodular and approximately monotone, or monotone and approximately submodular, functions with a size constraint, the approximation guarantees obtained by the GSEMO-C always reach the best known ones, which were previously obtained by the standard greedy algorithm~\cite{das2011submodular,horel2016maximization,krause2010submodular,krause2008near}. Note that the approximate guarantees here are achieved within polynomial time, while the GSEMO-C is actually an anytime algorithm and can find better solutions by running longer. If the running time is allowed to be infinite, the GSEMO-C can eventually find an optimal solution, since the mutation operator employed for reproduction is a global search operator leading to a positive probability of generating any solution in each iteration.

Friedrich and Neumann~\cite{friedrich2015maximizing} have proved that for maximizing monotone submodular functions with a size constraint, the GSEMO can achieve the approximation ratio of $(1-1/e)$, which is optimal in general~\cite{nemhauser1978best}. Without further assumptions or knowledge of the function, no polynomial time algorithm can provide a better approximation guarantee unless P$=$NP. Note that their result is generalized by our analysis for submodular and approximately monotone, or monotone and approximately submodular, functions. When the function is monotone submodular, the parameters characterizing the approximately monotone or submodular degree satisfy that $\epsilon=0$ and $\gamma=1$, and the approximation guarantees in Theorems~\ref{theo-apprx-monotone}-\ref{theo-nonsubmodular-3} all specialize to $1-1/e$, consistent with~\cite{friedrich2015maximizing}. Furthermore, our analysis may provide guidance under what conditions the GSEMO-C can have bounded approximation guarantees, even when the function is non-monotone or non-submodular. We have shown that the performance of the GSEMO-C is theoretically guaranteed for diverse applications with non-monotone or non-submodular objective functions, including sensor placement~\cite{krause2008near}, sparse regression~\cite{das2011submodular}, sparse support selection~\cite{elenberg2018restricted}, dictionary selection~\cite{krause2010submodular}, Bayesian experimental design~\cite{krause2008near} and determinantal function maximization~\cite{bian2017guarantees} (i.e., {\bf Corollaries~\ref{coro-application-1}-~\ref{coro-application-2}}). Our analytical results on general problem classes together with the previous ones~\cite{friedrich2015maximizing,reichel2010evolutionary} provide a theoretical explanation for the empirically good behaviors of EAs in diverse applications.

The rest of this paper is organized as follows. Sections~\ref{sec-problem} and~\ref{sec-algorithm} introduce the concerned problem classes and algorithm, respectively. Section~\ref{sec-problem-non-mono} presents the analysis for submodular function maximization with/without a size constraint. Section~\ref{sec-problem-non-sub} presents the analysis for monotone approximately submodular function maximization with a size constraint. Section~\ref{sec-conclusion} concludes the paper.

\section{Problem Classes}\label{sec-problem}

In this section, we introduce the problem classes studied in this paper. Let $\mathbb{R}$ and $\mathbb{R}^{+}$ denote the set of reals and non-negative reals, respectively. Given a finite non-empty set $V=\{v_1,v_2,\ldots,v_n\}$, we study the functions $f:2^V \rightarrow \mathbb{R}$ defined on subsets of $V$. A set function $f:2^V \rightarrow \mathbb{R}$ is monotone if for any $X \subseteq Y$, $f(X) \leq f(Y)$, which implies that adding more elements to a set never decreases the function value. Without loss of generality, we assume that monotone functions are normalized, i.e., $f(\emptyset)=0$. A set function $f$ is submodular~\cite{nemhauser1978analysis} if for any $X \subseteq Y \subseteq V$ and $v \notin Y$,
\begin{align}\label{def-submodular-1}
f(X \cup \{v\})-f(X) \geq f(Y \cup \{v\}) - f(Y);
\end{align}
or equivalently for any $X \subseteq Y \subseteq V$,
\begin{align}\label{def-submodular-2}
f(Y)-f(X) \leq \sum\nolimits_{v \in Y \setminus X} \big(f(X \cup \{v\})-f(X)\big).
\end{align}
Eq.~(\refeq{def-submodular-1}) intuitively represents the ``diminishing returns" property, i.e., adding an element to a set $X$ gives a larger benefit than adding the same element to a superset $Y$ of $X$. Eq.~(\refeq{def-submodular-2}) implies that the benefit by adding a set of elements to a set $X$ is smaller than the combined benefits of adding its individual elements to $X$. We assume that a set function $f$ is given by a value oracle, i.e., for a given subset $X$, an algorithm can query an oracle to obtain the value $f(X)$. In the following, let $\mathrm{OPT}$ denote the optimal function value.

\subsection{Submodular Function Maximization with/without a Size Constraint}

We consider the problem class of submodular function maximization, where the objective function is submodular, but not necessarily monotone. Both the situations without constraints as well as with a size constraint will be studied. Without loss of generality, we assume that the objective function $f$ is non-negative.

\begin{definition}[Non-monotone Submodular Function Maximization without Constraints]\label{def-Prob-nonmonotone-1}
Given a non-monotone and submodular function $f: 2^V \rightarrow \mathbb{R}^+$, to find a subset $X\subseteq V$ such that
$$
\arg \max\nolimits_{X \subseteq V}\quad f(X).
$$
\end{definition}

For the problem without constraints as presented in Definition~\ref{def-Prob-nonmonotone-1}, the goal is to maximize a non-monotone submodular set function. The best known approximation guarantee is $1/2$, which was achieved by the double greedy algorithm~\cite{buchbinder2015tight}. This problem generalizes many NP-hard combinatorial optimization problems, e.g., maximum cut~\cite{goemans1995improved}. Let $G=(V,E)$ be a graph with non-negative edge weights $w: E \rightarrow \mathbb{R}^+$, where $V$ and $E$ are the set of nodes and edges, respectively. For a subset $X$ of nodes, let $c(X)$ be the set of edges whose nodes are in $X$ and $V\setminus X$, respectively. The maximum cut problem is to find a subset $X$ of nodes maximizing the weighted cut function $\sum_{e \in c(X)} w(e)$, which is submodular but not monotone. More examples include maximum facility location~\cite{ageev19990}, variants of the maximum satisfiability problem~\cite{haastad2001some}, etc.

\begin{definition}[Approximately Monotone Submodular Function Maximization with a Size Constraint]\label{def-Prob-nonmonotone-2}
Given a submodular and approximately monotone function $f: 2^V \rightarrow \mathbb{R}^+$ and a budget $k$, to find a subset $X\subseteq V$ such that
$$
\arg \max\nolimits_{X \subseteq V}\quad f(X)\quad  \text{s.t.} \quad |X| \leq k.
$$
\end{definition}

For the problem with a size constraint as presented in Definition~\ref{def-Prob-nonmonotone-2}, the objective function, though not monotone, is required to be approximately monotone. The notion of ``approximate monotonicity"~\cite{krause2008near} was introduced to measure to what extent a general set function $f$ has the monotone property. As presented in Definition~\ref{def-approx-monotone}, a set function is $\epsilon$-approximately monotone implies that adding one element to a set decreases the function by at most $\epsilon$.

\begin{definition}[$\epsilon$-Approximate Monotonicity~\cite{krause2008near}]\label{def-approx-monotone}
Let $\epsilon\geq 0$. A set function $f: 2^V \rightarrow \mathbb{R}$ is $\epsilon$-approximately monotone if for any $X \subseteq V$ and $v \notin X$, $$f(X \cup \{v\}) \geq f(X)-\epsilon.$$
\end{definition}

It is easy to see that $f$ is monotone iff $\epsilon=0$. The standard greedy algorithm, which iteratively adds one element with the largest $f$ improvement until $k$ elements are selected, has been proved to achieve a subset $X$ with $f(X)\geq (1-1/e)\cdot (\mathrm{OPT}-k\epsilon)$~\cite{krause2008near}. A typical application is the sensor placement task, i.e., to select locations to install a limited number of sensors such that spatial phenomena can be monitored well. A common criterion to be maximized is the mutual information, which is submodular but not monotone. It has been shown~\cite{krause2008near} that a polynomial discretization level of locations can guarantee that the mutual information is $\epsilon$-approximately monotone. Note that there are applications (e.g., maximum entropy sampling~\cite{shewry1987maximum}) where the objective function is even not approximately monotone, i.e., $\epsilon$ is not well bounded.

\subsection{Monotone Approximately Submodular Function Maximization with a Size Constraint}\label{subsec-problem-non-sub}

Another concerned problem class is presented in Definition~\ref{def-Prob-nonsubmodular}. The goal is to find a subset with at most $k$ elements such that a given monotone and approximately submodular set function is maximized. Note that the situation without constraints is not considered, as it is trivial that an optimal solution is the whole set $V$ for monotone functions.

\begin{definition}[Monotone Approximately Submodular Function Maximization with a Size Constraint]\label{def-Prob-nonsubmodular}
Given a monotone and approximately submodular function $f: 2^V \rightarrow \mathbb{R}^+$ and a budget $k$, to find a subset $X\subseteq V$ such that
$$
\arg \max\nolimits_{X \subseteq V}\quad f(X)\quad  \text{s.t.} \quad |X| \leq k.
$$
\end{definition}

Several notions of ``approximate submodularity"~\cite{das2011submodular,horel2016maximization,krause2010submodular} were introduced to measure to what extent a set function $f$ has the submodular property. For each approximately submodular notion, the best known approximation guarantee was achieved by the standard greedy algorithm~\cite{das2011submodular,horel2016maximization,krause2010submodular}.

In~\cite{krause2010submodular}, the approximate submodularity as presented in Definition~\ref{def-approx-submodular-1} was defined based on the diminishing returns property, i.e., Eq.~(\refeq{def-submodular-1}). That is, the approximately submodular degree depends on how large a deviation of $\epsilon$ the diminishing returns property can hold with.

\begin{definition}[$\epsilon$-Diminishing Returns~\cite{krause2010submodular}]\label{def-approx-submodular-1}
Let $\epsilon\geq 0$. A set function $f: 2^V \rightarrow \mathbb{R}$ satisfies the $\epsilon$-diminishing returns property, if for any $X \subseteq Y \subseteq V$ and $v \notin Y$,
\begin{equation}
\begin{aligned}\label{eq-epsilon-diminishing}
f(X \cup \{v\})-f(X) \geq f(Y \cup \{v\})-f(Y)-\epsilon.
\end{aligned}
\end{equation}
\end{definition}

A set function $f$ satisfies the $\epsilon$-diminishing returns property implies that adding an element to a set $Y$ helps at most $\epsilon$ more than adding it to a subset $X$ of $Y$. It is easy to see that $f$ is submodular iff $\epsilon=0$. The standard greedy algorithm has been proved to find a subset $X$ with $f(X) \geq (1-1/e)\cdot(\mathrm{OPT}-k\epsilon)$~\cite{krause2010submodular}.

In~\cite{das2011submodular}, the submodularity ratio as presented in Definition~\ref{def-approx-submodular-2} was introduced to measure the closeness of a set function $f$ to submodularity.
\begin{definition}[Submodularity Ratio~\cite{das2011submodular}]\label{def-approx-submodular-2}
Let $f: 2^V \rightarrow \mathbb{R}$ be a set function. The submodularity ratio of $f$ with respect to a set $X \subseteq V$ and a parameter $l \geq 1$ is
$$
\gamma_{X,l}(f)=\min_{L \subseteq X, S: |S|\leq l, S \cap L =\emptyset} \frac{\sum_{v \in S} (f(L \cup \{v\})-f(L))}{f(L \cup S)-f(L)}.
$$
\end{definition}
Intuitively, the submodularity ratio captures how much more $f$ can increase by adding any set $S$ with at most $l$ elements to any subset $L$ of $X$, compared with the combined increment on $f$ by adding the individual elements of $S$ to $L$. It is easy to see from Eq.~(\refeq{def-submodular-2}) that $f$ is submodular iff $\gamma_{X,l}(f) = 1$ for any $X$ and $l$. When the meaning of $f$ is clear in the paper, we will omit $f$ and use $\gamma_{X,l}$ for short. The standard greedy algorithm has been proved to find a subset $X$ with $f(X) \geq (1-e^{-\gamma_{X,k}})\cdot \mathrm{OPT}$~\cite{das2011submodular}.

The above two notions of approximate submodularity are based on the equivalent statements, i.e., Eqs.~(\ref{def-submodular-1}) and~(\ref{def-submodular-2}), of submodularity, while in~\cite{horel2016maximization}, the approximate submodularity of a set function $f$ as presented in Definition~\ref{def-approx-submodular-3} was defined based on the closeness to other submodular functions.

\begin{definition}[$\epsilon$-Approximate Submodularity~\cite{horel2016maximization}]\label{def-approx-submodular-3}
Let $\epsilon\geq 0$. A set function $f: 2^V \rightarrow \mathbb{R}$ is $\epsilon$-approximately submodular if there exists a submodular set function $g$ such that $\forall X \subseteq V$, $$(1-\epsilon)\cdot g(X)\leq f(X)\leq (1+\epsilon)\cdot g(X).$$
\end{definition}

It is easy to see that $f$ is submodular iff $\epsilon=0$. The standard greedy algorithm has been proved to find a subset $X$ with $f(X) \geq \frac{1}{1+\frac{4k\epsilon }{(1-\epsilon)^2}}(1-e^{-1}(\frac{1-\epsilon}{1+\epsilon})^{2k})\cdot \mathrm{OPT}$~\cite{horel2016maximization}.

Note that a set function satisfying Eq.~(\refeq{eq-epsilon-diminishing}) was also originally said to be $\epsilon$-approximately submodular~\cite{krause2010submodular}. For a clearer presentation, we have renamed $\epsilon$-approximate submodularity to $\epsilon$-diminishing returns in Definition~\ref{def-approx-submodular-1}.

Next, we introduce five applications, i.e., sparse regression, sparse support selection, dictionary selection, Bayesian experimental design, and determinantal function maximization, that will be examined in this paper.

\subsubsection{Sparse Regression}

Sparse regression is to find a sparse approximation solution to the linear regression problem, where the solution vector can have only a few non-zero elements.

\begin{definition}[Sparse Regression~\cite{das2011submodular}]\label{def_sr} Given all observation variables $V=\{v_1,v_2,\ldots,v_n\}$, a predictor variable $z$ and a budget $k$, to find a set of at most $k$ observation variables maximizing the \emph{squared multiple correlation}~\citep{johnson2007applied}, i.e.,
$$
\mathop{\arg\max}\nolimits_{X \subseteq V} \left(R^2_{z,X}=\frac{\mathrm{Var}(z)-\mathrm{MSE}_{z,X}}{\mathrm{Var}(z)}\right) \quad \text{s.t.}\quad |X| \leq k,
$$
where $\mathrm{Var}(z)$ denotes the variance of $z$ and $\mathrm{MSE}_{z,X}=\min\nolimits_{\bm{\alpha} \in \mathbb{R}^{|X|}} \mathbb{E}[(z-\sum\nolimits_{i \in X} \alpha_i v_{i})^2]$ denotes the \emph{mean squared error}.
\end{definition}

Note that in the definition of $\mathrm{MSE}_{z,X}$, $X$ and its index set $\{i \mid v_i \in X\}$ are not distinguished for notational convenience. The objective function $R^2_{z,X}$, capturing the portion of the variance of $z$ explained by variables in $X$, is monotone but not necessarily submodular. Let $\mathbf{C}$ be the covariance matrix between all observation variables, and $\lambda_{\min}(\mathbf{C},m)$ be the smallest $m$-sparse eigenvalue of $\mathbf{C}$, i.e., the minimum eigenvalue of any $m \times m$ submatrix of $\mathbf{C}$. It has been proved~\cite{das2011submodular} that the submodularity ratio of $R^2_{z,X}$ can be lower bounded as $\gamma_{X,l} \geq \lambda_{\min}(\mathbf{C},|X|+l) \geq \lambda_{\min}(\mathbf{C},n)$.

\subsubsection{Sparse Support Selection}

Sparse support selection is a general sparsity constraint problem. The goal is to maximize general concave functions under sparsity constraints.

\begin{definition}[Sparse Support Selection~\cite{elenberg2018restricted}]\label{def_sss}
Given a ground set $V=\{v_1,v_2,\ldots,v_n\}$, a concave function $g: \mathbb{R}^n \rightarrow \mathbb{R}$ and a budget $k$, to find a subset $X \subseteq V$ such that
$$
\mathop{\arg\max}\nolimits_{X \subseteq V} \left(f(X)=\max_{\mathrm{supp}(\bm{s})\subseteq X}g(\bm{s})-g(\bm{0})\right) \quad \text{s.t.}\quad |X| \leq k,
$$
where $\mathrm{supp}(\bm{s})=\{v_i \mid s_i \neq 0\}$ denotes the support of $\bm{s} \in \mathbb{R}^n$.
\end{definition}

It is clear that sparse regression in Definition~\ref{def_sr} is a special case. More examples include low rank optimization~\cite{khanna2017approximation}, etc. Note that $g(\bm{0})$ is subtracted for normalization. The objective $f$ is monotone but not necessarily submodular. It has been proved~\cite{elenberg2018restricted} that when the concave function $g$ is $m$-strongly concave on all $(|X|+l)$-sparse vectors and $M$-smooth on all $(|X|+1)$-sparse vectors, the submodularity ratio of $f$ satisfies $\gamma_{X,l} \geq m/M$.

\subsubsection{Dictionary Selection}

Dictionary selection generalizes sparse regression to estimate multiple predictor variables.

\begin{definition}[Dictionary Selection~\cite{krause2010submodular}]\label{def.dictionary} Given all observation variables $V=\{v_1,v_2,\ldots,v_n\}$, multiple predictor variables $\{z_1,z_2,\ldots,z_m\}$ and two positive integers $k$ and $d$, to find a set of at most $k$ observation variables maximizing the \emph{average squared multiple correlation}, i.e.,
$$
\arg\max\nolimits_{X \subseteq V} \left(f(X)=\frac{1}{m}\sum\nolimits^m_{i=1} \max\nolimits_{ S \subseteq X, |S| \leq d} R^2_{z_i,S}\right)\quad \text{s.t.}\quad |X|\leq k.
$$
\end{definition}

The objective $f$ is monotone. It has been proved~\cite{krause2010submodular} that $f$ satisfies the $\epsilon$-diminishing returns property with $\epsilon \leq 4d\mu$, where $\mu$ denotes the coherence of $V$, i.e., the maximum absolute correlation between any pair of observation variables.

\subsubsection{Bayesian Experimental Design}

In Bayesian experimental design, the goal is to select observations to maximize the quality of parameter estimation. Krause \textit{et al.}~\cite{krause2008near} considered the Bayesian A-optimality objective function, in order to maximally reduce the variance of the posterior distribution over parameters in linear models. For a matrix $\mathbf{V} \in \mathbb{R}^{d\times n}$, let $\mathbf{V}_X \in \mathbb{R}^{d \times |X|}$ denote the submatrix of $\mathbf{V}$ with its columns indexed by $X \subseteq \{1,2,\ldots,n\}$.

\begin{definition}[Bayesian Experimental Design~\cite{krause2008near}]\label{def.bed}
Given an observation matrix $\mathbf{V}=[\bm{v}_1,\bm{v}_2,\ldots,\bm{v}_n] \in \mathbb{R}^{d\times n}$, a linear model $\bm{y}_X=\mathbf{V}^{\mathrm{T}}_{X} \bm{\theta}+\bm{w}$ and a budget $k$, where $\bm{\theta} \sim \mathcal{N}(0,\mathbf{\Lambda}^{-1})$, $\mathbf{\Lambda}=\beta^2\mathbf{I}_d$, the Gaussian noise $\bm{w} \sim \mathcal{N}(0,\sigma^2\mathbf{I}_{|X|})$, and $\mathbf{I}_j$ denotes the identity matrix of size $j$, to find a submatrix $\mathbf{V}_{X}$ of at most $k$ columns maximizing the Bayesian A-optimality objective function, i.e.,
$$
\arg\max\nolimits_{X \subseteq \{1,2,\ldots,n\}} \left(f(X)={\rm tr}(\mathbf{\Lambda}^{-1})-{\rm tr}((\mathbf{\Lambda}+\sigma^{-2}\mathbf{V}_X\mathbf{V}^{\mathrm{T}}_X)^{-1})\right)\quad \text{s.t.}\quad |X|\leq k,
$$
where $\rm{tr}(\cdot)$ denotes the trace of a matrix.
\end{definition}

Note that each $\bm{v}_i \in \mathbb{R}^{d}$ has been assumed to be normalized, i.e., $\|\bm{v}_i\|=1$. The objective $f$ is monotone, and the submodularity ratio satisfies $\gamma_{X,l} \geq \beta^2/(\|\mathbf{V}\|^2(\beta^2+\sigma^{-2}\|\mathbf{V}\|^2))$~\cite{bian2017guarantees}, where $\|\cdot\|$ denotes the spectral norm of a matrix.

\subsubsection{Determinantal Function Maximization}

In non-parametric learning, e.g., sparse Gaussian processes, the goal is to select a set of representative data points. Let $\mathbf{C} \in \mathbb{R}^{n \times n}$ be the covariance matrix parameterized by a positive definite kernel. Let $\mathbf{C}^X \in \mathbb{R}^{|X| \times |X|}$ denote the submatrix of $\mathbf{C}$ with its rows and columns indexed by $X \subseteq \{1,2,\ldots,n\}$. The determinantal function, $f(X)={\rm det}(\mathbf{I}_{|X|}+\sigma^{-2}\mathbf{C}^{X})$, is often involved in the objective functions of non-parametric learning, e.g.,~\cite{kulesza2012determinantal,herbrich2003fast}. Bian \textit{et al.}~\cite{bian2017guarantees} considered the problem of maximizing the determinantal function with a size constraint.

\begin{definition}[Determinantal Function Maximization~\cite{bian2017guarantees}]\label{def.dfm}
Given a data matrix $\mathbf{V}=[\bm{v}_1,\bm{v}_2,\ldots,\bm{v}_n] \in \mathbb{R}^{d\times n}$ with the covariance matrix $\mathbf{C} \in \mathbb{R}^{n \times n}$, and a budget $k$, to find a submatrix $\mathbf{V}_{X}$ of at most $k$ columns maximizing the determinantal function, i.e.,
$$
\arg\max\nolimits_{X \subseteq \{1,2,\ldots,n\}} \left(f(X)={\rm det}(\mathbf{I}_{|X|}+\sigma^{-2}\mathbf{C}^{X})\right)\quad \text{s.t.}\quad |X|\leq k,
$$
where $\sigma >0$.
\end{definition}

Though the logarithm of $f$ is monotone and submodular~\cite{krause2005near}, the determinantal function $f$ itself is not submodular. Let $\mathbf{A}$ denote $\mathbf{I}_{n}+\sigma^{-2}\mathbf{C}$. It has been proved~\cite{qian2018approximation} that $\gamma_{X,l} \geq (\lambda_n(\mathbf{A})-1)/((\lambda_1(\mathbf{A})-1)\prod^{n-1}_{i=1}\lambda_i(\mathbf{A}))$, where $\lambda_i(\cdot)$ denotes the $i$-th largest eigenvalue of a square matrix.

\section{Multi-objective Evolutionary Algorithms}\label{sec-algorithm}

To examine the performance of EAs optimizing the problem classes in Definitions~\ref{def-Prob-nonmonotone-1},~\ref{def-Prob-nonmonotone-2} and~\ref{def-Prob-nonsubmodular}, we consider a simple multi-objective EA called GSEMO-C, which is slightly modified from the algorithm GSEMO widely used in previous theoretical analyses~\cite{bian2018general,friedrich2010approximating,neumann2011computing,qian2013analysis}. The GSEMO generates a new solution (i.e., set) by bit-wise mutation in each iteration, whereas the GSEMO-C generates this new set as well as its complement in each iteration. Note that the letter ``C" in GSEMO-C denotes ``complement".

The GSEMO-C as presented in Algorithm~\ref{algo:GSEMO} is used for maximizing multi-objective pseudo-Boolean problems with $m$ objective functions $f_i: \{0,1\}^n \rightarrow \mathbb{R}$ ($1 \leq i\leq m$). Note that a pseudo-Boolean function $f: \{0,1\}^n \rightarrow \mathbb{R}$ naturally characterizes a set function $f: 2^{V} \rightarrow \mathbb{R}$, since a subset $X$ of $V$ can be naturally represented by a Boolean vector $\bm{x} \in \{0,1\}^n$, where the $i$-th bit $x_i=1$ means that $v_i \in X$, and $x_i=0$ means that $v_i \notin X$. Throughout the paper, we will not distinguish $\bm{x}\in \{0,1\}^n$ and its corresponding subset for notational convenience.

Before introducing the GSEMO-C, we first introduce some basic concepts in multi-objective maximization. Since the objectives to be maximized are usually conflicted, there is no canonical complete order on the solution space. The comparison between two solutions relies on the \emph{domination} relationship. For two solutions $\bm x$ and $\bm{x}'$, $\bm{x}$ \emph{weakly dominates} $\bm{x}'$ (i.e., $\bm{x}$ is \emph{better} than $\bm{x}'$, denoted by $\bm{x} \succeq \bm{x}'$) if $\forall 1 \leq i \leq m$, $f_i(\bm{x}) \geq f_i(\bm{x}')$; ${\bm{x}}$ \emph{dominates} $\bm{x}'$ (i.e., $\bm{x}$ is \emph{strictly better} than $\bm{x}'$, denoted by $\bm{x} \succ \bm{x}'$) if ${\bm{x}} \succeq \bm{x}'$ and $f_i(\bm{x}) > f_i(\bm{x}')$ for some $i$. But if neither $\bm{x}$ is better than $\bm{x}'$ nor $\bm{x}'$ is better than $\bm{x}$, we say that they are \emph{incomparable}. A solution is \emph{Pareto optimal} if no other solution dominates it. The set of objective vectors of all the Pareto optimal solutions constitutes the \emph{Pareto front}. The goal of multi-objective optimization is to find the Pareto front, that is, to find at least one corresponding solution for each objective vector in the Pareto front.

The procedure of the GSEMO-C is presented in Algorithm~\ref{algo:GSEMO}. Starting from a random solution (lines~1-2), it iteratively tries to improve the quality of the solutions in the population $P$ (lines~3-12). In each iteration, a new solution $\bm{x}'$ is generated by randomly flipping bits of an archived solution $\bm{x}$ selected from the current population $P$ (lines~4-5); the complementary set $\bm{x}''=V\setminus \bm{x}'$ of $\bm{x}'$ is also generated (line~6); these two newly generated solutions are then used to update the population $P$ (lines~7-11). In the updating procedure, if $\bm{y} \in \{\bm{x}',\bm{x}''\}$ is not dominated by (i.e., not strictly worse than) any previously archived solution (line~8), it will be added into $P$, and meanwhile those previously archived solutions weakly dominated by (i.e., worse than) $\bm{y}$ will be removed from $P$ (line~9). It is easy to see that the population $P$ will always contain a set of incomparable solutions due to the domination-based comparison.

\begin{algorithm}[t]\caption{GSEMO-C}
    Given $m$ pseudo-Boolean objective functions $f_1,f_2,\ldots,f_m$, where $f_i: \{0,1\}^n \rightarrow \mathbb{R}$, the GSEMO-C consists of the following steps:\label{algo:GSEMO}
    \begin{algorithmic}[1]
    \STATE Choose $\bm{x} \in \{0,1\}^n$ uniformly at random;
    \STATE $P \gets \{\bm x\}$;
    \STATE \textbf{repeat}
    \STATE \quad Choose $\bm x$ from $P$ uniformly at random;
    \STATE \quad Create $\bm{x}'$ by flipping each bit of $\bm x$ with probability $1/n$;
    \STATE \quad Create $\bm{x}'' \gets V \setminus \bm{x}'$;
    \STATE \quad \textbf{for} $\bm{y} \in \{\bm{x}',\bm{x}''\}$
    \STATE \qquad \textbf{if} \, {$\nexists \bm z \in P$ such that $\bm z \succ \bm {y}$} \,\textbf{then}
    \STATE \qquad \qquad $P \gets (P \setminus \{\bm z \in P \mid \bm {y} \succeq \bm z\}) \cup \{\bm {y}\}$
    \STATE \qquad \textbf{end if}
    \STATE \quad \textbf{end for}
    \STATE \textbf{until} some criterion is met
    \end{algorithmic}
\end{algorithm}

Compared with the GSEMO~\cite{bian2018general,friedrich2010approximating,neumann2011computing,qian2013analysis}, the GSEMO-C additionally performs line~6, i.e., generates the complement $\bm{x}''$ of the new solution $\bm{x}'$. Also, both $\bm{x}'$ and $\bm{x}''$, rather than only $\bm{x}'$, are used to update the population.

For optimizing the problems in Definitions~\ref{def-Prob-nonmonotone-1},~\ref{def-Prob-nonmonotone-2} and~\ref{def-Prob-nonsubmodular} by the GSEMO-C, each problem is transformed into a bi-objective maximization problem
$$
\arg\max\nolimits_{\bm{x} \in \{0,1\}^n} \quad  (f_1(\bm{x}),f_2(\bm{x})),
$$
where $f_1(\bm x) = f(\bm x)$ and $f_2(\bm x) = -|\bm{x}|$. That is, the GSEMO-C is to maximize the objective function $f$ and minimize the subset size $|\bm{x}|$ simultaneously. Note that $|\bm{x}|=\sum^n_{i=1}x_i$ denotes the number of 1-bits of a solution $\bm{x}$. When the GSEMO-C terminates after running a number of iterations, the best solution w.r.t. the original single-objective problem in the resulting population $P$ will be returned. For the problem in Definition~\ref{def-Prob-nonmonotone-1}, the solution with the largest $f$ value in $P$ (i.e., $\arg\max_{\bm{x} \in P} f(\bm{x})$) will be returned. For the problem in Definitions~\ref{def-Prob-nonmonotone-2} and~\ref{def-Prob-nonsubmodular}, the solution with the largest $f$ value satisfying the size constraint in $P$ (i.e., $\arg\max_{\bm{x} \in P, |\bm{x}|\leq k} f(\bm{x})$) will be returned. The running time of the GSEMO-C is measured by the number of fitness evaluations until the best solution w.r.t. the original single-objective problem in the population reaches some approximation guarantee for the first time. Since only the new solutions $\bm{x}'$ and $\bm{x}''$ need to be evaluated in each iteration of the GSEMO-C, the number of fitness evaluations is just the double of the number of iterations of the GSEMO-C.

Note that multi-objective optimization here is just an intermediate process, which has been shown helpful for solving some single-objective combinatorial optimization problems~\cite{friedrich2010approximating,neumann2006minimum,neumann2011computing,qian.ijcai15}. We still focus on the quality of the best solution w.r.t. the original single-objective problem, in the population found by the GSEMO-C, rather than the quality of the population w.r.t. the transformed bi-objective optimization problem.

\section{Analysis on Submodular Function Maximization}\label{sec-problem-non-mono}

In this section, we theoretically analyze the performance of the GSEMO-C for maximizing submodular, but not necessarily monotone, functions.

\subsection{Without Constraints}

First, we consider the problem class in Definition~\ref{def-Prob-nonmonotone-1}, i.e., maximizing non-monotone submodular functions without constraints. We prove in Theorem~\ref{theo-nonmonotone} that the GSEMO-C can achieve a constant approximation ratio of nearly $1/3$ in polynomial expected time.

\begin{theorem}\label{theo-nonmonotone}
For maximizing a non-monotone submodular function without constraints, the expected running time of the GSEMO-C until finding a solution $\bm{x}$ with $f(\bm{x}) \geq (\frac{1}{3}-\frac{\epsilon}{n}) \cdot \mathrm{OPT}$ is $O(\frac{1}{\epsilon}n^4\log n)$, where $\epsilon>0$.
\end{theorem}

The proof relies on Lemma~\ref{lemma-nonmonotone-mid}, which shows that it is always possible to improve a solution by inserting or deleting one element until a good approximation has been achieved. This lemma is extracted from Lemma~3.4 in~\cite{feige2011maximizing}.

\begin{lemma}[\cite{feige2011maximizing}]\label{lemma-nonmonotone-mid}
Let $\bm{x}\in \{0,1\}^n$ be a solution such that no solution $\bm{x}'$ with the objective value $f(\bm{x}')>(1+\frac{\epsilon}{n^2})\cdot f(\bm{x})$ can be achieved by inserting one element into $\bm{x}$ or deleting one element from $\bm{x}$, where $\epsilon >0$. Then $\max\{f(\bm{x}),f(V\setminus \bm{x})\} \geq (\frac{1}{3}-\frac{\epsilon}{n})\cdot \mathrm{OPT}$.
\end{lemma}

Inspired from the proof of Theorem~4 in~\cite{friedrich2015maximizing}, the intuition of our proof is to follow the behavior of the local search algorithm~\cite{feige2011maximizing}, which iteratively tries to improve a solution by inserting or deleting one element.

\begin{myproof}{Theorem~\ref{theo-nonmonotone}}
We divide the optimization process into three phases: (1) starts from an initial random solution and finishes after finding the all-0s solution $\bm{0}$ (i.e., $\emptyset$); (2) starts after phase~(1) and finishes after finding a solution with the objective value at least $\mathrm{OPT}/n$; (3) starts after phase~(2) and finishes after finding a solution with the desired approximation guarantee. We analyze the expected running time of each phase, respectively, and then sum up them to get an upper bound on the total expected running time of the GSEMO-C.

For phase (1), we consider the minimum number of 1-bits of the solutions in the population $P$, denoted by $J_{\min}$. That is, $J_{\min}=\min\{|\bm{x}| \mid \bm{x} \in P\}$. Assume that currently $J_{\min}=i>0$, and let $\bm{x}$ be the corresponding solution, i.e., $|\bm{x}|=i$. It is easy to see that $J_{\min}$ cannot increase because $\bm{x}$ cannot be weakly dominated by a solution with more 1-bits. In each iteration of the GSEMO-C, to decrease $J_{\min}$, it is sufficient to select $\bm{x}$ in line~4 of Algorithm~\ref{algo:GSEMO} and flip only one 1-bit of $\bm{x}$ in line~5. This is because the newly generated solution $\bm{x}'$ now has the smallest number of 1-bits (i.e., $|\bm{x}'|=i-1$) and no solution in $P$ can dominate it; thus it will be included into $P$. Let $P_{\max}$ denote the largest size of $P$ during the run of the GSEMO-C. The probability of selecting $\bm{x}$ in line~4 of Algorithm~\ref{algo:GSEMO} is $\frac{1}{|P|} \geq \frac{1}{P_{\max}}$ due to uniform selection, and the probability of flipping only one 1-bit of $\bm{x}$ in line~5 is $\frac{i}{n}(1-\frac{1}{n})^{n-1} \geq \frac{i}{en}$, since $\bm{x}$ has $i$ 1-bits. Thus, the probability of decreasing $J_{\min}$ by at least 1 in each iteration of the GSEMO-C is at least $\frac{i}{enP_{\max}}$. Note that $J_{\min} \leq n$. We can then get that the expected number of iterations of phase (1) (i.e., $J_{\min}$ reaches 0) is at most
$$
\sum^{n}_{i=1} \frac{enP_{\max}}{i}=O(nP_{\max}\log n).
$$
Note that the solution $\bm{0}$ will always be kept in $P$ once generated, since it has the smallest subset size 0 and no other solution can weakly dominate it.

For phase (2), it is sufficient that in one iteration of the GSEMO-C, the solution $\bm{0}$ is selected in line~4, and only a specific 0-bit corresponding to the best single element $v^*$ (i.e., $v^* \in \arg\max_{v \in V} f(\{v\})$) is flipped in line~5. That is, the solution $\{v^*\}$ is generated. Since the objective function $f$ is submodular and non-negative, we easily have $f(\{v^*\})\geq \mathrm{OPT}/n$. After generating the solution $\{v^*\}$, it will be used to update the population $P$, which makes $P$ always contain a solution $\bm{z}$ weakly dominating $\{v^*\}$, i.e., $f(\bm{z}) \geq f(\{v^*\}) \geq \mathrm{OPT}/n$ and $|\bm{z}|\leq |\{v^*\}|=1$. Thus, we only need to analyze the expected number of iterations of the GSEMO-C until generating the solution $\{v^*\}$. Since the probability of selecting $\bm{0}$ in line~4 of the GSEMO-C is at least $\frac{1}{P_{\max}}$ and the probability of flipping only a specific 0-bit in line~5 is $\frac{1}{n}(1-\frac{1}{n})^{n-1}\geq \frac{1}{en}$, the expected number of iterations of phase (2) is $O(nP_{\max})$.

As in~\cite{feige2011maximizing}, we call a solution $\bm{x}$ a $(1+\alpha)$-approximate local optimum if $f(\bm{x} \setminus \{v\}) \leq (1+\alpha)\cdot f(\bm{x})$ for any $v \in \bm{x}$ and $f(\bm{x} \cup \{v\}) \leq (1+\alpha)\cdot f(\bm{x})$ for any $v \notin \bm{x}$. By Lemma~\ref{lemma-nonmonotone-mid}, we know that a $(1+\frac{\epsilon}{n^2})$-approximate local optimum $\bm{x}$ satisfies $\max\{f(\bm{x}),f(V\setminus \bm{x})\} \geq (\frac{1}{3}-\frac{\epsilon}{n})\cdot \mathrm{OPT}$. For phase~(3), we thus only need to analyze the expected number of iterations until generating a $(1+\frac{\epsilon}{n^2})$-approximate local optimum $\bm{x}'$. This is because both $\bm{x}'$ and $V\setminus \bm{x}'$ will be used to update the population $P$, and then for either one of $\bm{x}'$ and $V\setminus \bm{x}'$, $P$ will always contain one solution weakly dominating it, which implies that $\max\{f(\bm{x})\mid \bm{x} \in P\} \geq \max\{f(\bm{x}'),f(V\setminus \bm{x}')\} \geq (\frac{1}{3}-\frac{\epsilon}{n})\cdot \mathrm{OPT}$. We then consider the largest $f$ value of the solutions in the population $P$, denoted by $J_{\max}$. That is, $J_{\max}=\max\{f(\bm{x}) \mid \bm{x} \in P\}$. After phase (2), $J_{\max}\geq \mathrm{OPT}/n$, and let $\bm{x}$ be the corresponding solution, i.e., $f(\bm{x})=J_{\max}$. It is obvious that $J_{\max}$ cannot decrease, because $\bm{x}$ cannot be weakly dominated by a solution with a smaller $f$ value. As long as $\bm{x}$ is not a $(1+\frac{\epsilon}{n^2})$-approximate local optimum, we know that a new solution $\bm{x}'$ with $f(\bm{x}')>(1+\frac{\epsilon}{n^2})f(\bm{x})=(1+\frac{\epsilon}{n^2})J_{\max}$ can be generated through selecting $\bm{x}$ in line~4 of Algorithm~\ref{algo:GSEMO} and flipping only one specific 1-bit (i.e., deleting one specific element from $\bm{x}$) or one specific 0-bit (i.e., adding one specific element into $\bm{x}$) in line~5, the probability of which is at least $\frac{1}{P_{\max}}\cdot \frac{1}{n}(1-\frac{1}{n})^{n-1}\geq \frac{1}{enP_{\max}}$. Since $\bm{x}'$ now has the largest $f$ value and no other solution in $P$ can dominate it, it will be included into $P$. Thus, $J_{\max}$ can increase by at least a factor of $(1+\frac{\epsilon}{n^2})$ with probability at least $\frac{1}{enP_{\max}}$ in each iteration. Such an increase on $J_{\max}$ is called a successful step. Thus, a successful step needs at most $enP_{\max}$ expected number of iterations. It is also easy to see that until generating a $(1+\frac{\epsilon}{n^2})$-approximate local optimum, the number of successful steps is at most $\log_{1+\frac{\epsilon}{n^2}} \frac{\mathrm{OPT}}{\mathrm{OPT}/n} = O(\frac{1}{\epsilon}n^2\log n)$. Thus, the expected number of iterations of phase (3) is at most
$$
enP_{\max} \cdot O\left(\frac{1}{\epsilon}n^2\log n\right) =  O\left(\frac{1}{\epsilon}n^3P_{\max}\log n\right).
$$

From the procedure of the GSEMO-C, we know that the solutions maintained in $P$ must be incomparable. Thus, each value of one objective can correspond to at most one solution in $P$. Because the second objective $f_2(\bm{x})=-|\bm{x}|$ can only belong to $\{0,-1,\ldots,-n\}$, we have $P_{\max} \leq n+1$. Hence, the expected running time of the GSEMO-C for finding a solution with the objective function value at least $(\frac{1}{3}-\frac{\epsilon}{n})\cdot \mathrm{OPT}$ is
$$O(nP_{\max}\log n)+O(nP_{\max})+O\left(\frac{1}{\epsilon}n^3P_{\max}\log n\right)=O\left(\frac{1}{\epsilon}n^4\log n\right).$$
\end{myproof}

Note that in parallel with our work, Friedrich \textit{et al.}~\cite{friedrich2018heavy} analyzed the performance of the (1+1)-EA using a new and novel mutation operator for solving this problem class. The mutation operator is a heavy-tailed mutation operator, which samples $m \in \{1,2,\ldots,n\}$ according to a power-law distribution, and then flips $m$ bits of a solution chosen uniformly at random. The power-law distribution has a parameter $\beta>1$, that can be chosen arbitrarily close to 1. They proved that the (1+1)-EA can achieve an approximation ratio of $(\frac{1}{3}-\frac{\epsilon}{n})$ in $O(\frac{1}{\epsilon}n^3\log\frac{n}{\epsilon}+n^{\beta})$ expected running time.

\subsection{With a Size Constraint}

Next, we consider the problem class in Definition~\ref{def-Prob-nonmonotone-2}, i.e., maximizing submodular and approximately monotone functions with a size constraint. As in previous analyses (e.g.,~\cite{buchbinder2014submodular,friedrich2018greedy}), we may assume that there is a set $D$ of $k$ ``dummy" elements whose marginal contribution to any set is 0, i.e., for any $X \subseteq V$, $f(X)=f(X\setminus D)$. Theorem~\ref{theo-apprx-monotone} gives the approximation guarantee of the GSEMO-C.

\begin{theorem}\label{theo-apprx-monotone}
For maximizing a submodular function $f$ with a size constraint $k$, where $f$ is $\epsilon$-approximately monotone as in Definition~\ref{def-approx-monotone}, the expected running time of the GSEMO-C until finding a solution $\bm{x}$ with $|\bm{x}| \leq k$ and $f(\bm{x}) \geq (1-1/e) \cdot (\mathrm{OPT}-k\epsilon)$ is $O(n^2(\log n+k))$.
\end{theorem}

The proof relies on Lemma~\ref{lemma-apprx-monotone}, that for any $\bm{x} \in \{0,1\}^n$ with $|\bm{x}|<k$, there always exists another element, the inclusion of which can bring an improvement on $f$ roughly proportional to the current distance to the optimum.

\begin{lemma}\label{lemma-apprx-monotone}
Assume that a set function $f$ is submodular and $\epsilon$-approximately monotone as in Definition~\ref{def-approx-monotone}. For any $\bm{x} \in \{0,1\}^n$ with $|\bm{x}| < k$, there exists one element $v \notin \bm{x}$ such that
\begin{align}\label{eq-mid-4}
& f(\bm{x} \cup \{v\})-f(\bm{x}) \geq \frac{1}{k} (\mathrm{OPT}-f(\bm{x}))-\epsilon,
\end{align}
where $k$ is the size constraint.
\end{lemma}
\begin{proof}
Let $\bm{x}^*$ be an optimal solution, i.e., $f(\bm{x}^*)=\mathrm{OPT}$. We denote the elements in $\bm{x}\setminus \bm{x}^*$ by $u^*_1,u^*_2,\ldots,u^*_m$, where $m=|\bm{x}\setminus \bm{x}^*|$. Note that $m< k$ as $|\bm{x}|< k$. Because $f$ is $\epsilon$-approximately monotone, we have
\begin{align}\label{eq-mid-1}
f(\bm{x}^* \cup \bm{x})=f(\bm{x}^* \cup \{u^*_1,\ldots,u^*_m\}) &\geq f(\bm{x}^* \cup \{u^*_1,\ldots,u^*_{m-1}\})-\epsilon\\
&\geq \cdots \geq f(\bm{x}^*)-m\epsilon \geq f(\bm{x}^*)-k\epsilon,
\end{align}
where the first three inequalities hold by Definition~\ref{def-approx-monotone}.

We denote the elements in $\bm{x}^*\setminus \bm{x}$ by $v^*_1,v^*_2,\ldots,v^*_l$, where $l=|\bm{x}^*\setminus \bm{x}| \leq k$. Then, we have
\begin{align}\label{eq-mid-3}
f(\bm{x}^*)-f(\bm{x})-k\epsilon &\leq f(\bm{x} \cup \bm{x}^*)-f(\bm{x})\\
&=f(\bm{x} \cup \{v^*_1,\ldots,v^*_l\})-f(\bm{x})\\
&= \sum^{l}_{j=1} \left(f(\bm{x} \cup \{v^*_1,\ldots,v^*_j\})-f(\bm{x} \cup \{v^*_1,\ldots,v^*_{j-1}\})\right)\\
&\leq \sum^{l}_{j=1} \left(f(\bm{x} \cup \{v^*_j\})-f(\bm{x})\right),
\end{align}
where the first inequality holds by Eq.~(\refeq{eq-mid-1}), the first equality holds by the definition of $\bm{x}^*\setminus \bm{x}$, and the last inequality holds by Eq.~(\refeq{def-submodular-1}) since $f$ is submodular. Let $v^*=\arg \max_{v \in V \setminus \bm{x}} f(\bm{x} \cup \{v\})$. Eq.~(\refeq{eq-mid-3}) implies that
$$
f(\bm{x}^*)-f(\bm{x})-k\epsilon \leq l \left(f(\bm{x} \cup \{v^*\})-f(\bm{x})\right).
$$
Due to the existence of $k$ dummy elements and $|\bm{x}|<k$, there must exist one dummy element $v \notin \bm{x}$ satisfying $f(\bm{x} \cup \{v\})-f(\bm{x}) = 0$; this implies that $f(\bm{x} \cup \{v^*\})-f(\bm{x}) \geq 0$. As $l \leq k$, we have
$$
f(\bm{x}^*)-f(\bm{x})-k\epsilon \leq k \left(f(\bm{x} \cup \{v^*\})-f(\bm{x})\right),
$$
leading to
$$
f(\bm{x} \cup \{v^*\})-f(\bm{x}) \geq \frac{1}{k} (\mathrm{OPT}-f(\bm{x}))-\epsilon.
$$
\end{proof}

Inspired from the proof of Theorem~2 in~\cite{friedrich2015maximizing}, our proof idea is to follow the behavior of the standard greedy algorithm, which iteratively adds one element with the currently largest improvement on $f$.

\begin{myproof}{Theorem~\ref{theo-apprx-monotone}}
We divide the optimization process into two phases: (1) starts from an initial random solution and finishes after finding the special solution $\bm{0}$; (2) starts after phase (1) and finishes after finding a solution with the desired approximation guarantee. As the analysis of phase~(1) in the proof of Theorem~\ref{theo-nonmonotone}, we know that the population $P$ will contain the solution $\bm{0}$ after $O(nP_{\max}\log n)$ iterations in expectation.

For phase (2), we consider a quantity $J_{\max}$, which is defined as $$
J_{\max}=\max\left\{j \in \{0,1,\ldots,k\} \mid \exists \bm{x} \in P: |\bm{x}| \leq j \wedge f(\bm{x}) \geq \left(1-\left(1-\frac{1}{k}\right)^j\right) \cdot (\mathrm{OPT}-k\epsilon)\right\}.$$ That is, $J_{\max}$ denotes the maximum value of $j \in \{0,1,\ldots,k\}$ such that in the population $P$, there exists a solution $\bm{x}$ with $|\bm{x}| \leq j$ and $f(\bm{x}) \geq (1-(1-\frac{1}{k})^j) \cdot (\mathrm{OPT}-k\epsilon)$. We analyze the expected number of iterations until $J_{\max}=k$, which implies that there exists one solution $\bm{x}$ in $P$ satisfying that $|\bm{x}| \leq k$ and $f(\bm{x}) \geq (1-(1-\frac{1}{k})^k) \cdot (\mathrm{OPT}-k\epsilon) \geq (1-1/e) \cdot (\mathrm{OPT}-k\epsilon)$. That is, the desired approximation guarantee is reached.

The current value of $J_{\max}$ is at least 0, since the population $P$ contains the solution $\bm{0}$, which will always be kept in $P$ once generated. Assume that currently $J_{\max}=i <k$. Let $\bm{x}$ be a corresponding solution with the value $i$, i.e., $|\bm{x}|\leq i$ and $f(\bm{x}) \geq (1-(1-\frac{1}{k})^i) \cdot (\mathrm{OPT}-k\epsilon)$. It is easy to see that $J_{\max}$ cannot decrease because cleaning $\bm{x}$ from $P$ (line~9 of Algorithm~\ref{algo:GSEMO}) implies that $\bm{x}$ is weakly dominated by a newly generated solution $\bm{y}$, which must satisfy that $|\bm{y}| \leq |\bm{x}|$ and $f(\bm{y})\geq f(\bm{x})$. By Lemma~\ref{lemma-apprx-monotone}, we know that flipping one specific 0-bit of $\bm{x}$ (i.e., adding a specific element) can generate a new solution $\bm{x}'$, which satisfies $f(\bm{x}')-f(\bm{x}) \geq \frac{1}{k} (\mathrm{OPT}-f(\bm{x}))-\epsilon$. Then, we have
$$
f(\bm{x}') \geq \left(1-\frac{1}{k}\right)f(\bm{x})+\frac{1}{k}\cdot \mathrm{OPT} -\epsilon\geq \left(1-\left(1-\frac{1}{k}\right)^{i+1}\right) \cdot (\mathrm{OPT}-k\epsilon),
$$
where the last inequality is derived by $f(\bm{x}) \geq (1-(1-\frac{1}{k})^i) \cdot (\mathrm{OPT}-k\epsilon)$. Since $|\bm{x}'|=|\bm{x}|+1 \leq i+1$, $\bm{x}'$ will be included into $P$; otherwise, $\bm{x}'$ must be dominated by one solution in $P$ (line~8 of Algorithm~\ref{algo:GSEMO}), and this implies that $J_{\max}$ has already been larger than $i$, contradicting with the assumption $J_{\max}=i$. After including $\bm{x}'$, $J_{\max} \geq i+1$. Thus, $J_{\max}$ can increase by at least 1 in one iteration with probability at least $\frac{1}{P_{\max}} \cdot \frac{1}{n}(1-\frac{1}{n})^{n-1} \geq \frac{1}{enP_{\max}}$, where $\frac{1}{P_{\max}}$ is a lower bound on the probability of selecting $\bm{x}$ in line~4 of Algorithm~\ref{algo:GSEMO} and $\frac{1}{n}(1-\frac{1}{n})^{n-1}$ is the probability of flipping a specific bit of $\bm{x}$ while keeping other bits unchanged in line~5. This implies that it needs at most $enP_{\max}$ expected number of iterations to increase $J_{\max}$. Thus, after at most $k \cdot enP_{\max}$ iterations in expectation, $J_{\max}$ must have reached $k$.

As the proof of Theorem~\ref{theo-nonmonotone}, we know that $P_{\max} \leq n+1$. Thus, by summing up the expected running time of two phases, we get that the expected running time of the GSEMO-C for finding a solution $\bm{x}$ with $|\bm{x}|\leq k$ and $f(\bm{x})\geq (1-1/e) \cdot (\mathrm{OPT}-k\epsilon)$ is $O(nP_{\max}\log n +knP_{\max})=O(n^2(\log n+k))$.
\end{myproof}

Note that the approximation guarantee, i.e., $f(\bm{x}) \geq (1-1/e)\cdot (\mathrm{OPT}-k\epsilon)$, by the GSEMO-C reaches the best known one, which was previously obtained by the standard greedy algorithm~\cite{krause2008near}. Particularly, when the objective function is monotone, the parameter $\epsilon$ in Definition~\ref{def-approx-monotone} equals 0, and thus the approximation ratio becomes $1-1/e$, which is optimal in general~\cite{nemhauser1978best}, and also consistent with the previous result in~\cite{friedrich2015maximizing}. For the application of sensor placement with the mutual information as the objective function, which is submodular but not necessarily monotone, the GSEMO-C has a bounded approximation guarantee, because the mutual information can be guaranteed to be $\epsilon$-approximately monotone~\cite{krause2008near}, where $\epsilon$ depends on the discretization level of locations.

\section{Analysis on Monotone Approximately Submodular Function Maximization}\label{sec-problem-non-sub}

In this section, we analyze the performance of the GSEMO-C for maximizing monotone and approximately submodular functions with a size constraint, which has various applications as introduced in Section~\ref{subsec-problem-non-sub}. We prove the polynomial-time approximation guarantee of the GSEMO-C w.r.t. each notion of approximate submodularity in Definitions~\ref{def-approx-submodular-1}-\ref{def-approx-submodular-3}, respectively.

\subsection{$\epsilon$-Diminishing Returns}

First, we consider the case that the objective function satisfies the $\epsilon$-diminishing returns property in Definition~\ref{def-approx-submodular-1}. Theorem~\ref{theo-nonsubmodular-1} gives the approximation guarantee of the GSEMO-C.

\begin{theorem}\label{theo-nonsubmodular-1}
For maximizing a monotone function $f$ with a size constraint $k$, where $f$ satisfies the $\epsilon$-diminishing returns property as in Definition~\ref{def-approx-submodular-1}, the expected running time of the GSEMO-C until finding a solution $\bm{x}$ with $|\bm{x}| \leq k$ and $f(\bm{x}) \geq (1-1/e) \cdot (\mathrm{OPT}-k\epsilon)$ is $O(n^2(\log n+k))$.
\end{theorem}

The proof relies on Lemma~\ref{lemma-nonsubmodular-1}, which states that any $\bm{x} \in \{0,1\}^n$ can be improved by at least roughly $(\mathrm{OPT}-f(\bm{x}))/k$ through adding a specific element. The proof of Lemma~\ref{lemma-nonsubmodular-1} is similar to that of Lemma~\ref{lemma-apprx-monotone}. The main difference is that the two inequalities in Eq.~(\refeq{eq-mid-3}) utilize the $\epsilon$-approximately monotone and the diminishing returns properties, respectively, whereas that in Eq.~(\refeq{eq-mid-2}) utilize the monotone and the $\epsilon$-diminishing returns properties, respectively.

\begin{lemma}\label{lemma-nonsubmodular-1}
Assume that a set function $f$ is monotone and satisfies the $\epsilon$-diminishing returns property as in Definition~\ref{def-approx-submodular-1}. For any $\bm{x} \in \{0,1\}^n$, there exists one element $v \notin \bm{x}$ such that
\begin{align}\label{eq-mid-5}
&f(\bm{x} \cup \{v\})-f(\bm{x}) \geq \frac{1}{k} (\mathrm{OPT}-f(\bm{x}))-\epsilon,
\end{align}
where $k$ is the size constraint.
\end{lemma}
\begin{proof}
Let $\bm{x}^*$ be an optimal solution, i.e., $f(\bm{x}^*)=\mathrm{OPT}$. We denote the elements in $\bm{x}^*\setminus \bm{x}$ by $v^*_1,v^*_2,\ldots,v^*_l$, where $|\bm{x}^*\setminus \bm{x}|=l \leq k$. Then, we have
\begin{align}\label{eq-mid-2}
f(\bm{x}^*)-f(\bm{x}) &\leq f(\bm{x} \cup \bm{x}^*)-f(\bm{x})\\
&=f(\bm{x} \cup \{v^*_1,\ldots,v^*_l\})-f(\bm{x})\\
&= \sum^{l}_{j=1} \left(f(\bm{x} \cup \{v^*_1,\ldots,v^*_j\})-f(\bm{x} \cup \{v^*_1,\ldots,v^*_{j-1}\})\right)\\
&\leq \sum^{l}_{j=1} \left(f(\bm{x} \cup \{v^*_j\})-f(\bm{x}) +\epsilon\right),
\end{align}
where the first inequality holds by the monotonicity of $f$, the first equality holds by the definition of $\bm{x}^*\setminus \bm{x}$, and the last inequality is derived by Definition~\ref{def-approx-submodular-1} since $f$ satisfies the $\epsilon$-diminishing returns property. Let $v^*=\arg \max_{v \in \bm{x}^* \setminus \bm{x}} f(\bm{x} \cup \{v\})$. Then, we have
$$
f(\bm{x} \cup \{v^*\})-f(\bm{x}) \geq \frac{1}{l} (f(\bm{x}^*)-f(\bm{x})) -\epsilon \geq \frac{1}{k} (\mathrm{OPT}-f(\bm{x}))-\epsilon.
$$
\end{proof}

Thus, the proof of Theorem~\ref{theo-nonsubmodular-1} can be accomplished in the same way as that of Theorem~\ref{theo-apprx-monotone}. This is because the proof of Theorem~\ref{theo-apprx-monotone} utilizes a quantity $J_{\max}$ based on Eq.~(\refeq{eq-mid-4}), while Eq.~(\refeq{eq-mid-4}) still holds here as Eq.~(\refeq{eq-mid-5}) in Lemma~\ref{lemma-nonsubmodular-1}.

Note that this approximation guarantee, i.e., $f(\bm{x}) \geq (1-1/e) \cdot (\mathrm{OPT}-k\epsilon)$, obtained by the GSEMO-C reaches the best known one, which was previously obtained by the standard greedy algorithm~\cite{krause2010submodular}. Particularly, when the objective function is submodular, the diminishing returns property holds, i.e., $\epsilon=0$, and thus the approximation ratio reaches the optimal one, $1-1/e$.

For the application of dictionary selection in Definition~\ref{def.dictionary} where the objective function $f$ is monotone but not necessarily submodular, as $f$ satisfies the $\epsilon$-diminishing returns property with $\epsilon \leq 4d\mu$~\cite{krause2010submodular}, we have:

\begin{corollary}\label{coro-application-1}
For dictionary selection in Definition~\ref{def.dictionary}, the expected running time of the GSEMO-C until finding a solution $\bm{x}$ with $|\bm{x}| \leq k$ and $f(\bm{x}) \geq (1-1/e) \cdot (\mathrm{OPT}-4dk\mu)$ is $O(n^2(\log n+k))$, where $\mu$ denotes the coherence of $V$, i.e., the maximum absolute correlation between any pair of observation variables.
\end{corollary}

\subsection{Submodularity Ratio}

Next, we prove the approximation guarantee of the GSEMO-C w.r.t. the submodularity ratio presented in Definition~\ref{def-approx-submodular-2}.

\begin{theorem}\label{theo-nonsubmodular-2}
For maximizing a monotone function $f$ with a size constraint $k$, where $f$ is not necessarily submodular, the expected running time of the GSEMO-C until finding a solution $\bm{x}$ with $|\bm{x}| \leq k$ and $f(\bm{x}) \geq (1-e^{-\gamma_{\min}}) \cdot \mathrm{OPT}$ is $O(n^2(\log n+k))$, where $\gamma_{\min}=\min_{\bm{x}:|\bm{x}|=k-1}\gamma_{\bm{x},k}$ and $\gamma_{\bm{x},k}$ is the submodularity ratio of $f$ w.r.t. $\bm{x}$ and $k$ as in Definition~\ref{def-approx-submodular-2}.
\end{theorem}

The proof relies on Lemma~\ref{lemma-nonsubmodular-2}, which shows that any $\bm{x} \in \{0,1\}^n$ can be improved by adding a specific element such that the increment on $f$ is proportional to $(\mathrm{OPT}-f(\bm{x}))$ and depends on the submodularity ratio $\gamma_{\bm{x},k}$.

\begin{lemma}[\cite{qian2016parallel}]\label{lemma-nonsubmodular-2}
Assume that a set function $f$ is monotone but not necessarily submodular. For any $\bm{x} \in \{0,1\}^n$, there exists one element $v \notin \bm{x}$ such that
\begin{align}\label{eq-mid-6}
&f(\bm{x} \cup \{v\})-f(\bm{x}) \geq \frac{\gamma_{\bm{x},k}}{k} (\mathrm{OPT}-f(\bm{x})),
\end{align}
where $k$ is the size constraint, and $\gamma_{\bm{x},k}$ is the submodularity ratio of $f$ w.r.t. $\bm{x}$ and $k$ as in Definition~\ref{def-approx-submodular-2}.
\end{lemma}

The proof of Theorem~\ref{theo-nonsubmodular-2} is similar to that of Theorem~\ref{theo-apprx-monotone}. The main difference is that a different inductive inequality on $f$ is used in the definition of the quantity $J_{\max}$, as Eq.~(\ref{eq-mid-4}) in Lemma~\ref{lemma-apprx-monotone} changes to Eq.~(\refeq{eq-mid-6}) in Lemma~\ref{lemma-nonsubmodular-2}. For concise illustration, we will mainly show the difference in the proof of Theorem~\ref{theo-nonsubmodular-2}.

\begin{myproof}{Theorem~\ref{theo-nonsubmodular-2}}
The proof is similar to that of Theorem~\ref{theo-apprx-monotone}. We use a different $J_{\max}$, which is defined as
$$
J_{\max}=\max\left\{j \in \{0,1,\ldots,k\} \mid \exists \bm{x} \in P: |\bm{x}| \leq j \wedge f(\bm{x}) \geq \left(1-\left(1-\frac{\gamma_{\min}}{k}\right)^j\right) \cdot \mathrm{OPT}\right\}.$$
It is easy to verify that $J_{\max}=k$ implies that the desired approximation guarantee is reached, because there must exist one solution $\bm{x}$ in $P$ satisfying that $|\bm{x}| \leq k$ and $f(\bm{x}) \geq (1-(1-\frac{\gamma_{\min}}{k})^k) \cdot \mathrm{OPT} \geq (1-e^{-\gamma_{\min}}) \cdot \mathrm{OPT}$. Assume that currently $J_{\max}=i<k$ and $\bm{x}$ is a corresponding solution, i.e., $|\bm{x}| \leq i$ and $f(\bm{x})\geq (1-(1-\frac{\gamma_{\min}}{k})^i) \cdot \mathrm{OPT}$. We then only need to show that flipping one specific 0-bit of $\bm{x}$ can generate a new solution $\bm{x}'$ with $f(\bm{x}')\geq (1-(1-\frac{\gamma_{\min}}{k})^{i+1}) \cdot \mathrm{OPT}$. By Lemma~\ref{lemma-nonsubmodular-2}, we know that flipping one specific 0-bit of $\bm{x}$ can generate a new solution $\bm{x}'$, which satisfies $f(\bm{x}')-f(\bm{x}) \geq \frac{\gamma_{\bm{x},k}}{k} (\mathrm{OPT}-f(\bm{x}))$. Then, we have
\begin{align}
f(\bm{x}') &\geq \left(1-\frac{\gamma_{\bm{x},k}}{k}\right)f(\bm{x})+\frac{\gamma_{\bm{x},k}}{k}\cdot \mathrm{OPT} \\
&\geq \left(1-\left(1-\frac{\gamma_{\bm{x},k}}{k}\right)\left(1-\frac{\gamma_{\min}}{k}\right)^{i}\right)\cdot \mathrm{OPT}\\
&\geq \left(1-\left(1-\frac{\gamma_{\min}}{k}\right)^{i+1}\right)\cdot \mathrm{OPT},
\end{align}
where the second inequality holds by $f(\bm{x})\geq (1-(1-\frac{\gamma_{\min}}{k})^i) \cdot \mathrm{OPT}$, and the last holds by $\gamma_{\bm{x},k} \geq \gamma_{\min}$, which can be derived from $|\bm{x}|<k$ and $\gamma_{\bm{x},k}$ decreasing with $\bm{x}$. Thus, the theorem holds.
\end{myproof}

Note that it has been proved that the standard greedy algorithm can find a subset $\bm{x}$ with $|\bm{x}|=k$ and $f(\bm{x}) \geq (1-e^{-\gamma_{\bm{x},k}}) \cdot \mathrm{OPT}$~\cite{das2011submodular}. Thus, Theorem~\ref{theo-nonsubmodular-2} shows that the GSEMO-C can achieve nearly this best known approximation guarantee. Particularly, when the objective function is submodular, the submodularity ratio in Definition~\ref{def-approx-submodular-2} satisfies $\forall\bm{x},l: \gamma_{\bm{x},l}=1$, and thus the approximation ratio, i.e., $1-e^{-\gamma_{\min}}$, by the GSEMO-C reaches the optimal one, $1-1/e$.

For the application of sparse regression in Definition~\ref{def_sr} where the objective function $R^2_{z,\bm{x}}$ is monotone but not necessarily submodular, because the submodularity ratio of $R^2_{z,\bm{x}}$ can be lower bounded as $\gamma_{\bm{x},l} \geq \lambda_{\min}(\mathbf{C},|\bm{x}|+l) \geq \lambda_{\min}(\mathbf{C},n)$~\cite{das2011submodular}, implying $\gamma_{\min}=\min_{\bm{x}:|\bm{x}|=k-1}\gamma_{\bm{x},k}\geq \lambda_{\min}(\mathbf{C},2k-1) \geq \lambda_{\min}(\mathbf{C},n)$, we have:

\begin{corollary}
For sparse regression in Definition~\ref{def_sr}, the expected running time of the GSEMO-C until finding a solution $\bm{x}$ with $|\bm{x}| \leq k$ and $f(\bm{x}) \geq (1-e^{-\lambda_{\min}(\mathbf{C},2k-1)}) \cdot \mathrm{OPT}\geq (1-e^{-\lambda_{\min}(\mathbf{C},n)}) \cdot \mathrm{OPT}$ is $O(n^2(\log n+k))$, where $\lambda_{\min}(\mathbf{C},m)$ denotes the smallest $m$-sparse eigenvalue of the covariance matrix $\mathbf{C}$ between all observation variables.
\end{corollary}

For the application of sparse support selection in Definition~\ref{def_sss} where the objective function $f(\bm{x})=\max_{\mathrm{supp}(\bm{s})\subseteq \bm{x}}g(\bm{s})-g(\bm{0})$ is monotone but not necessarily submodular, the submodularity ratio of $f$ satisfies $\gamma_{\bm{x},l} \geq m/M$, when the concave function $g$ is $m$-strongly concave on all $(|\bm{x}|+l)$-sparse vectors and $M$-smooth on all $(|\bm{x}|+1)$-sparse vectors~\cite{elenberg2018restricted}. Thus, $\gamma_{\min}=\min_{\bm{x}:|\bm{x}|=k-1}\gamma_{\bm{x},k} \geq m/M$, when $g$ is $m$-strongly concave on all $(2k-1)$-sparse vectors and $M$-smooth on all $k$-sparse vectors. We have:

\begin{corollary}
For sparse support selection in Definition~\ref{def_sss} where the concave function $g$ is $m$-strongly concave on all $(2k-1)$-sparse vectors and $M$-smooth on all $k$-sparse vectors, the expected running time of the GSEMO-C until finding a solution $\bm{x}$ with $|\bm{x}| \leq k$ and $f(\bm{x}) \geq (1-e^{-m/M}) \cdot \mathrm{OPT}$ is $O(n^2(\log n+k))$.
\end{corollary}

For the application of Bayesian experimental design in Definition~\ref{def.bed} where the objective function $f$ is monotone but not necessarily submodular, because the submodularity ratio of $f$ satisfies $\forall \bm{x},l: \gamma_{\bm{x},l} \geq \beta^2/(\|\mathbf{V}\|^2(\beta^2+\sigma^{-2}\|\mathbf{V}\|^2))$~\cite{bian2017guarantees}, implying $\gamma_{\min}\geq \beta^2/(\|\mathbf{V}\|^2(\beta^2+\sigma^{-2}\|\mathbf{V}\|^2))$, we have:

\begin{corollary}
For Bayesian experimental design in Definition~\ref{def.bed}, the expected running time of the GSEMO-C until finding a solution $\bm{x}$ with $|\bm{x}| \leq k$ and $f(\bm{x}) \geq (1-e^{-\beta^2/(\|\mathbf{V}\|^2(\beta^2+\sigma^{-2}\|\mathbf{V}\|^2))}) \cdot \mathrm{OPT}$ is $O(n^2(\log n+k))$.
\end{corollary}

For the application of determinantal function maximization in Definition~\ref{def.dfm} where the objective function $f$ is monotone but not necessarily submodular, because the submodularity ratio of $f$ satisfies $\forall \bm{x},l: \gamma_{\bm{x},l} \geq (\lambda_n(\mathbf{A})-1)/((\lambda_1(\mathbf{A})-1)\prod^{n-1}_{i=1}\lambda_i(\mathbf{A}))$~\cite{qian2018approximation}, implying $\gamma_{\min}\geq (\lambda_n(\mathbf{A})-1)/((\lambda_1(\mathbf{A})-1)\prod^{n-1}_{i=1}\lambda_i(\mathbf{A}))$, we have:

\begin{corollary}
For determinantal function maximization in Definition~\ref{def.dfm}, the expected running time of the GSEMO-C until finding a solution $\bm{x}$ with $|\bm{x}| \leq k$ and $f(\bm{x}) \geq (1-e^{-(\lambda_n(\mathbf{A})-1)/((\lambda_1(\mathbf{A})-1)\prod^{n-1}_{i=1}\lambda_i(\mathbf{A}))}) \cdot \mathrm{OPT}$ is $O(n^2(\log n+k))$, where $\mathbf{A}=\mathbf{I}_{n}+\sigma^{-2}\mathbf{C}$ and $\lambda_i(\mathbf{A})$ denotes the $i$-th largest eigenvalue of $\mathbf{A}$.
\end{corollary}

\subsection{$\epsilon$-Approximate Submodularity}

Finally, we consider the case that the objective function is $\epsilon$-approximately submodular as in Definition~\ref{def-approx-submodular-3}. Theorem~\ref{theo-nonsubmodular-3} gives the approximation guarantee of the GSEMO-C.

\begin{theorem}\label{theo-nonsubmodular-3}
For maximizing a monotone function $f$ with a size constraint $k$, where $f$ is $\epsilon$-approximately submodular as in Definition~\ref{def-approx-submodular-3}, the expected running time of the GSEMO-C until finding a solution $\bm{x}$ with $|\bm{x}| \leq k$ and $f(\bm{x}) \geq \frac{1}{1+\frac{2k\epsilon}{1-\epsilon}}(1-(1-\frac{1}{k})^k(\frac{1-\epsilon}{1+\epsilon})^k) \cdot \mathrm{OPT} \geq \frac{1}{1+\frac{2k\epsilon}{1-\epsilon}}(1-e^{-1}(\frac{1-\epsilon}{1+\epsilon})^k) \cdot \mathrm{OPT}$ is $O(n^2(\log n+k))$.
\end{theorem}

The proof relies on the following lemma, which shows that any $\bm{x} \in \{0,1\}^n$ can be improved by adding a specific element $v$ such that $f(\bm{x} \cup \{v\})-\frac{1-\epsilon}{1+\epsilon}f(\bm{x})$ is proportional to the current distance to the optimum, i.e., $\mathrm{OPT}-f(\bm{x})$.

\begin{lemma}\label{lemma-nonsubmodular-3}
Assume that a set function $f$ is monotone and $\epsilon$-approximately submodular as in Definition~\ref{def-approx-submodular-3}. For any $\bm{x} \in \{0,1\}^n$, there exists one element $v \notin \bm{x}$ such that
\begin{align}\label{eq-mid-7}
&f(\bm{x} \cup \{v\})-\frac{1-\epsilon}{1+\epsilon}f(\bm{x}) \geq \frac{1-\epsilon}{k(1+\epsilon)} (\mathrm{OPT}-f(\bm{x})),
\end{align}
where $k$ is the size constraint.
\end{lemma}
\begin{proof}
Let $\bm{x}^*$ be an optimal solution, i.e., $f(\bm{x}^*)=\mathrm{OPT}$. Let $v^*=\arg \max_{v \in \bm{x}^* \setminus \bm{x}} f(\bm{x} \cup \{v\})$. As $f$ is $\epsilon$-approximately submodular as in Definition~\ref{def-approx-submodular-3}, we use $g$ to denote one corresponding submodular function satisfying that for all $\bm{x} \in \{0,1\}^n$, $(1-\epsilon)g(\bm{x}) \leq f(\bm{x}) \leq (1+\epsilon)g(\bm{x})$. Then, we have
\begin{align}
g(\bm{x}^* \cup \bm{x})-g(\bm{x}) &\leq \sum_{v \in \bm{x}^*\setminus \bm{x}} \big(g(\bm{x} \cup \{v\})-g(\bm{x})\big)\\
&\leq \sum_{v \in \bm{x}^*\setminus \bm{x}} \left(\frac{1}{1-\epsilon}f(\bm{x} \cup \{v\})-g(\bm{x})\right)\\
&\leq k\left(\frac{1}{1-\epsilon}f(\bm{x} \cup \{v^*\})-g(\bm{x})\right),
\end{align}
where the first inequality holds by the submodularity of $g$ (i.e., Eq.~(\ref{def-submodular-2})), the second inequality holds by $(1-\epsilon)g(\bm{x}) \leq f(\bm{x})$ for any $\bm{x}$, and the last inequality holds by the definition of $v^*$ and $|\bm{x}^*|\leq k$. By reordering the terms, we get
$$
f(\bm{x} \cup \{v^*\}) \geq \frac{1-\epsilon}{k}g(\bm{x}^* \cup \bm{x})+\left(1-\frac{1}{k}\right)(1-\epsilon)g(\bm{x}).
$$
Because $g(\bm{x}) \geq \frac{1}{1+\epsilon}f(\bm{x})$ and $g(\bm{x}^* \cup \bm{x}) \geq \frac{1}{1+\epsilon}f(\bm{x}^* \cup \bm{x}) \geq \frac{1}{1+\epsilon}f(\bm{x}^*)=\frac{1}{1+\epsilon}\mathrm{OPT}$, where the last inequality holds by the monotonicity of $f$, we have
$$
f(\bm{x} \cup \{v^*\}) \geq \frac{1-\epsilon}{k(1+\epsilon)}\mathrm{OPT}+\left(1-\frac{1}{k}\right)\frac{1-\epsilon}{1+\epsilon}f(\bm{x}).
$$
By reordering the terms, the lemma holds.
\end{proof}

The proof of Theorem~\ref{theo-nonsubmodular-3} is also similar to that of Theorem~\ref{theo-apprx-monotone}, except that a different inductive inequality on $f$ is used in the definition of the quantity $J_{\max}$, as Eq.~(\ref{eq-mid-4}) in Lemma~\ref{lemma-apprx-monotone} changes to Eq.~(\refeq{eq-mid-7}) in Lemma~\ref{lemma-nonsubmodular-3}.

\begin{myproof}{Theorem~\ref{theo-nonsubmodular-3}}
The proof is similar to that of Theorem~\ref{theo-apprx-monotone}. We use a different $J_{\max}$, which is defined as
$$
J_{\max}=\max\left\{j \in \{0,1,\ldots,k\} \mid \exists \bm{x} \in P: |\bm{x}| \leq j \wedge f(\bm{x}) \geq \frac{1}{1+\frac{2k\epsilon}{1-\epsilon}}\left(1-\left(1\!-\!\frac{1}{k}\right)^j\left(\frac{1\!-\!\epsilon}{1\!+\!\epsilon}\right)^j\right) \cdot \mathrm{OPT}\right\}.$$
It is easy to verify that $J_{\max}=k$ implies that the desired approximation guarantee is reached. Assume that currently $J_{\max}=i<k$ and $\bm{x}$ is a corresponding solution, i.e., $|\bm{x}| \leq i$ and $f(\bm{x})\geq \frac{1}{1+\frac{2k\epsilon}{1-\epsilon}}(1-(1-\frac{1}{k})^i(\frac{1-\epsilon}{1+\epsilon})^i) \cdot \mathrm{OPT}$. We then only need to show that flipping one specific 0-bit of $\bm{x}$ can generate a new solution $\bm{x}'$ with $f(\bm{x}')\geq \frac{1}{1+\frac{2k\epsilon}{1-\epsilon}}(1-(1-\frac{1}{k})^{i+1}(\frac{1-\epsilon}{1+\epsilon})^{i+1}) \cdot \mathrm{OPT}$. By Lemma~\ref{lemma-nonsubmodular-3}, we know that flipping one specific 0-bit of $\bm{x}$ can generate a new solution $\bm{x}'$, which satisfies $f(\bm{x}')-\frac{1-\epsilon}{1+\epsilon}f(\bm{x}) \geq \frac{1-\epsilon}{k(1+\epsilon)} (\mathrm{OPT}-f(\bm{x}))$. Then, we have
$$
f(\bm{x}') \geq \left(1-\frac{1}{k}\right)\frac{1-\epsilon}{1+\epsilon}f(\bm{x})+\frac{1-\epsilon}{k(1+\epsilon)}\cdot \mathrm{OPT} \geq \frac{1}{1+\frac{2k\epsilon}{1-\epsilon}}\left(1-\left(1-\frac{1}{k}\right)^{i+1}\left(\frac{1-\epsilon}{1+\epsilon}\right)^{i+1}\right)\cdot \mathrm{OPT},
$$
where the second inequality is derived by applying $f(\bm{x}) \geq \frac{1}{1+\frac{2k\epsilon}{1-\epsilon}}(1-(1-\frac{1}{k})^i(\frac{1-\epsilon}{1+\epsilon})^i) \cdot \mathrm{OPT}$. Thus, the theorem holds.
\end{myproof}

Note that the standard greedy algorithm obtains the best known approximation guarantee, i.e., $f(\bm{x}) \geq \frac{1}{1+\frac{4k\epsilon}{(1-\epsilon)^2}}(1-(1-\frac{1}{k})^k(\frac{1-\epsilon}{1+\epsilon})^{2k}) \cdot \mathrm{OPT}$~\cite{horel2016maximization}. Compared with this, the approximation guarantee, i.e., $f(\bm{x}) \geq \frac{1}{1+\frac{2k\epsilon}{1-\epsilon}}(1-(1-\frac{1}{k})^k(\frac{1-\epsilon}{1+\epsilon})^k) \cdot \mathrm{OPT}$, of the GSEMO-C shown in Theorem~\ref{theo-nonsubmodular-3} is slightly better, because
\begin{align}
&\frac{1}{1+\frac{2k\epsilon}{1-\epsilon}}\left(1-\left(1-\frac{1}{k}\right)^{k}\left(\frac{1-\epsilon}{1+\epsilon}\right)^{k}\right)=\frac{1-\epsilon}{k(1+\epsilon)}\cdot \sum^{k-1}_{i=0} \left(\left(1-\frac{1}{k}\right)\frac{1-\epsilon}{1+\epsilon}\right)^{i}\\
&\geq \frac{(1-\epsilon)^2}{k(1+\epsilon)^2}\cdot \sum^{k-1}_{i=0}  \left(\left(1-\frac{1}{k}\right)\left(\frac{1-\epsilon}{1+\epsilon}\right)^{2}\right)^{i}= \frac{1}{1+\frac{4k\epsilon}{(1-\epsilon)^2}}\left(1-\left(1-\frac{1}{k}\right)^k\left(\frac{1-\epsilon}{1+\epsilon}\right)^{2k}\right).
\end{align}
Particularly, when the objective function is submodular, the parameter $\epsilon$ in Definition~\ref{def-approx-submodular-3} equals 0, and thus the approximation ratio of the GSEMO-C reaches the optimal one, $1-1/e$. When $\epsilon \leq 1/k$, we have
$$
\frac{1}{1+\frac{2k\epsilon}{1-\epsilon}}\left(1-\left(1-\frac{1}{k}\right)^{k}\left(\frac{1-\epsilon}{1+\epsilon}\right)^{k}\right) \geq \frac{1}{1+\frac{2}{1-1/k}}\left(1-\frac{1}{e}\right)\geq \frac{1}{5}\left(1-\frac{1}{e}\right),
$$
where the last inequality holds by $k\geq 2$; thus, the GSEMO-C can still achieve a constant approximation ratio, as shown below.

\begin{corollary}\label{coro-application-2}
For maximizing a monotone function $f$ with a size constraint $k$, where $f$ is $\epsilon$-approximately submodular with $\epsilon \leq 1/k$, the expected running time of the GSEMO-C until finding a solution $\bm{x}$ with $|\bm{x}| \leq k$ and $f(\bm{x}) \geq (1/5)(1-1/e) \cdot \mathrm{OPT}$ is $O(n^2(\log n+k))$.
\end{corollary}

\section{Conclusion}\label{sec-conclusion}

This paper theoretically studies the approximation performance of EAs for solving the general classes of combinatorial optimization problems, i.e., maximizing submodular functions with/without a size constraint and maximizing monotone approximately submodular functions with a size constraint. We prove that within polynomial expected running time, a simple multi-objective EA called GSEMO-C can achieve good approximation guarantees for any concerned problem class. These results may help to provide a theoretical explanation for the empirically good performance of EAs in various applications. A question that will be examined in the future is whether simple single-objective EAs such as the (1+1)-EA can achieve good approximation guarantees on the concerned problem classes, which has been partially addressed recently~\cite{friedrich2018heavy}. It is also interesting to study the performance of EAs under more complicated constraints, e.g., matroid and knapsack constraints~\cite{lee2009non}.

\section{Acknowledgments}

The authors want to thank the associate editor and anonymous reviewers for their helpful comments and suggestions. C. Qian, Y. Yu and K. Tang were supported by the NSFC (61603367, 61672478, 61876077). X. Yao was supported by the Program for Guangdong Introducing Innovative and Enterpreneurial Teams (2017ZT07X386) and Shenzhen Peacock Plan (KQTD2016112514355531). Z.-H. Zhou was supported by the National Key R\&D Program of China (2018YFB1004300) and Collaborative Innovation Center of Novel Software Technology and Industrialization.

\bibliography{aij19-submodularMOEA}
\bibliographystyle{abbrvnat}

\end{document}